\documentclass{article}

\usepackage{arxiv}

\usepackage[utf8]{inputenc} % allow utf-8 input
\usepackage[T1]{fontenc}    % use 8-bit T1 fonts
\usepackage{hyperref}       % hyperlinks
\usepackage{url}            % simple URL typesetting
\usepackage{booktabs}       % professional-quality tables
\usepackage{amsfonts}       % blackboard math symbols
\usepackage{nicefrac}       % compact symbols for 1/2, etc.
\usepackage{microtype}      % microtypography
\usepackage{lipsum}		% Can be removed after putting your text content
\usepackage{graphicx}
\usepackage{natbib}
\usepackage{doi}
\usepackage{microtype}
\usepackage{graphicx}
\usepackage{subfigure}
\usepackage{booktabs} 
\usepackage{amsfonts,amssymb}
\usepackage{mathptmx}
\usepackage{inconsolata}
\usepackage{amsmath}
\usepackage{amssymb}
\usepackage{mathtools}
\usepackage{amsthm}
\theoremstyle{plain}
\newtheorem{theorem}{Theorem}[section]

\newtheorem{lemma}[theorem]{Lemma}

\theoremstyle{definition}
\newtheorem{definition}[theorem]{Definition}

\theoremstyle{remark}

\usepackage{algorithm}
\usepackage{algorithmic}

\title{Does Debiasing Inevitably Degrade the Model Performance}

\author{ Yiran Liu \\
	Tsinghua University\\
	\texttt{liu-yr21@mails.tsinghua.edu.cn} \\
	\And
	Xiao Liu \\
	Tsinghua University\\
	\texttt{liuxiao21@mails.tsinghua.edu.cn} \\
	\And
	Haotian Chen \\
	Fudan University\\
	\texttt{htchen18@fudan.edu.cn} \\
	\And
	Yang Yu \\
	Tsinghua University\\
	\texttt{yangyu1@tsinghua.edu.cn} \\
}

\renewcommand{\shorttitle}{\textit{arXiv} Template}

\begin{document}
\maketitle

\begin{abstract}
	Gender bias in language models has attracted significant attention due to its potential to undermine social justice. However, current debiasing methods often result in performance degradation, whose mechanism is not well understood. We analyzed the causality structure that generates gender bias and discovered the mechanism that debiasing degrades the model performance. According to the analysis, we demonstrate that mitigating the gender bias \textbf{does not} inevitably cause performance degradation. We discover that only mitigating the gender bias caused by the embedding errors will not cause performance degradation. Analyzing the causality structure also enables us to demonstrate that the embedding-errors-induced gender bias can be separately mitigated by a low-cost fine-tuning approach. We develop a self-training fine-tuning approach to correct gender bias induced by the errors of the embedding layer. Our method is able to correct the embedding error when the gender-mutual subspace is unknown and non-Euclidean. Numerical experiments demonstrate the effectiveness of our approach to debiasing and model-performance maintenance. For instance, our method that only fine-tunes the embedding layer is able to achieve similar debiasing effectiveness and better model performance than fine-tuning the whole model. 
\end{abstract}

% keywords can be removed
%\keywords{First keyword \and Second keyword \and More}

\section{Introduction}
\vskip -0.1in
It is unclear whether the degradation of model performance is an inevitable consequence of debiasing approaches for correcting gender bias in pre-trained language models. Pre-trained language models have been highly successful in a variety of natural language processing tasks and are widely used in practice. However, several studies have revealed the presence of gender bias in pre-trained models \cite{bolukbasi2016man,zhaoetal2018learning,bhardwaj2021investigating,vig2020investigating}. For example, gender bias has been identified in the use of pre-trained models for online advertising~\cite{sweeney2013discrimination}, automatic resume filtering systems~\cite{cowgill2018bias,deshpande2020mitigating}, and automatic criminal sentencing~\cite{dressel2018accuracy}. This evidence of bias has raised concerns about the potential risks of using these models in real-world applications, such as recruitment and education. After \citet{Reuters2018Amazon} noticed that the automatic resume filtering adopted by Amazon discriminated against female candidates, the system was abandoned. Thus, many studies have explored fine-tuning methods of mitigating gender biases in pre-trained language models. As a result, many researchers have explored fine-tuning methods for mitigating gender bias in pre-trained language models. 

However, recent discussions suggest that current debiasing approaches are associated with a decline in model performance \cite{barikeri-etal-2021-redditbias,meadeetal2022empirical}. It is necessary to examine the possibility and method of easing the dilemma between model performance and gender bias.
%我们发现通过基于因果干预的可解释的方法能够有效地在纠正模型偏见的同时保护模型的性能。对当前广泛使用的transformer结构的语言模型而言，the sources of the gender bias can only the errors in two architectures of the pre-trained model: the word embedding and the transformer.
We discovered that using a method based on causal intervention can effectively reduce model bias and maintain performance. For transformer-structured language models, the source of gender bias is limited to errors in the pre-trained model's word embedding and transformer architectures.
% Here, we proposed a causal framework to examine the origin of the dilemma. We positioned the sources of the gender bias can only the errors in two architectures of the pre-trained model: the word embedding and the transformer. 
 %Those errors and their interactions lead the gender bias through three latent mechanisms. All the current debiasing methods intervene in one or two of the three mechanisms. The discoveries of the three mechanisms manifested the origin of the dilemma. 
 %对于
 %Among the three mechanisms, two of them are related to the model transformer while the other one is not. The current debiasing approaches all intervened in those two transformer-dependent mechanisms and thus cause performance degradation. 

%We demonstrate that intervening in the explicable mechanism is able to mitigate the bias while maintaining the performance. In fact, errors in word embedding can directly cause gender bias. Correcting that error can gain double dividends. In light of the above observation, we develop a performance-maintained debiasing method. While we notice that the correct embedding is unknown, we develop a gradient-based approach to infer the latent ground truth. We apply our method to debiasing GPT-2. Experimental results show that our method can effectively reduce bias and stably protect the performance of the model.
We  demonstrate that intervening in the explicable mechanisms of a language model is possible to reduce bias while preserving performance. Our findings indicate that errors in the word embedding can directly lead to gender bias, and that correcting these errors will not have a significant impact on the performance of the model. Based on this observation, we developed a debiasing method that is able to maintain performance. As the correct embedding is unknown, we developed a gradient-based approach to infer the latent ground truth. We applied this method to debias GPT-2 and found that it effectively reduced bias while maintaining the performance of the model.

Our work is an experiment using the causal-analysis approach to slove the dilemma between gender bias and model performance. Our contributions are summarized as follows. First, We correct the gender bias of the language model by adjusting the word embedding of the language model, which provides a new way to correct the bias. Second, We positioned the two architectures and three mechanisms that originate the gender bias. We demonstrate the possibility of double dividends. Third, We proposed a double-dividend fine-tuning debiasing method equipped with a causality-detection function. %Our method can significantly mitigate gender 

\section{Related Work}
\vskip -0.1in
Current debiasing research can generally be grouped into four categories. However, studies by \citet{barikeri-etal-2021-redditbias} and \cite{meadeetal2022empirical} have found that the first three types of methods tend to lead to performance degradation.
\vskip -0.05in
\textbf{Data Augmentation.} Data augmentation is a well-established debiasing method for reducing gender bias. \citet{zhaoetal2018learning} were the first to use this method by creating a gender-balanced dataset through gender-swapping, to train an unbiased model. Since then, data augmentation has been applied to various NLP tasks, such as knowledge graph building \cite{mitchell2019model} and machine translation \cite{stanovskyetal2019evaluating}.
\vskip -0.05in
\textbf{Removing Gender Subspace.} Removing the gender subspace is another popular debiasing method, which was proposed after \citet{bolukbasi2016man} demonstrated the link between word embeddings and gender bias. This method aims to remove the gender subspace in contextualized language models by subtracting the embedding's projection on a hypothetical gender subspace \cite{ravfogeletal2020null, liangetal2020towards}. However, this method's effectiveness is limited as it relies on the assumption that the gender subspace is Euclidean, which has been called into question by recent research.
\vskip -0.05in
\textbf{Gender-Equality Regularizer.} Several regularization methods have been developed to eliminate gender bias. Some of these methods focus on addressing the imbalance in the training data by adding regularizers \cite{bordia-bowman-2019-identifying,qian-etal-2019-reducing,lauscher2020general}. Other researchers such as \citet{barikeri-etal-2021-redditbias} proposed new loss functions for both data augmentation and gender-subspace removal. Additionally, \citet{zhang2018mitigating} proposed a generative adversarial approach for bias mitigation.
\vskip -0.05in
\textbf{Causal Inference.} 
%目前已经有一些利用反事实和因果干预的方法在静态词嵌入中进行纠偏的方法。
At present, there are some methods to mitigate gender bias in the static word embedding by counterfactual and causal intervention
\cite{yang2020causal,shin2020neutralizing}.
%但是,据我们所知,还没有人在transformer语言模型上使用基于因果的方法纠正偏见。
However, as far as we know, no one has implemented a causality-based debiasing method on the transformer language model.
%With the causal graph, causal inference helps us find the causal relationship between variables, evaluate the causal effect of variables on the results, and intervene in the system \cite{glymour2016causal}. \citet{vig2020investigating} used the causal inference to analyze the properties of gender bias in large pre-trained language models.

\section{A Causal View on Debiasing Dilemma}
\subsection{Causal Framework of Gender Bias Origination}
%在预训练过程中，优化算法会根据训练数据确定语言模型的参数。由于训练数据的不平衡，会导致模型在训练过程中学到错误的参数，产生社会歧视。其中有一类重要的错误重要是词嵌入在性别子空间上发生了偏差，这在许多研究中表明会直接产生模型的性别歧视。例如

\begin{figure*}[th!]
     \centering
     %\begin{minipage}{0.99\textwidth}
     %    \centering
     %    \includegraphics[width=0.9\textwidth]{le3.png}
         %\caption{}
         %\label{fig:pos_acc_cnndm}
     %\end{minipage}
     \begin{minipage}{0.26\textwidth}
         \centering
         \includegraphics[width=0.9\textwidth]{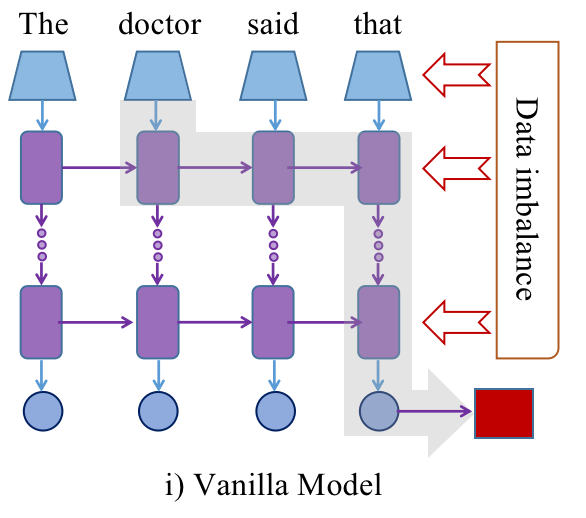}
         %\caption{}
         %\label{fig:pos_acc_cnndm}
     \end{minipage}
     %\hfill
     \begin{minipage}{0.26\textwidth}
         \centering
         \includegraphics[width=0.9\textwidth]{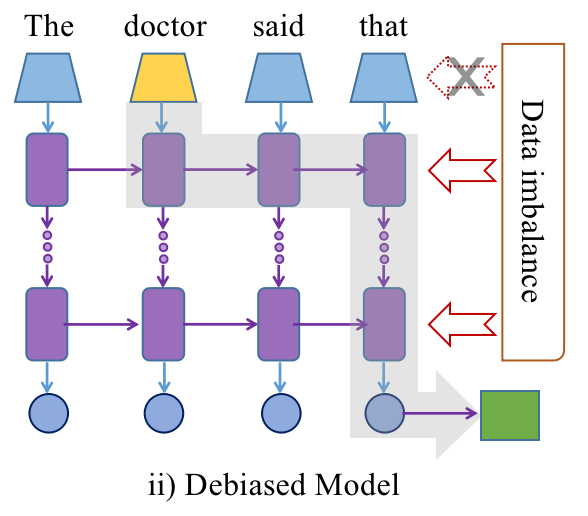}
         %\caption{}
         %\label{fig:pos_acc_samsum}
     \end{minipage}
     \begin{minipage}{0.26\textwidth}
         \centering
         \includegraphics[width=0.8\textwidth]{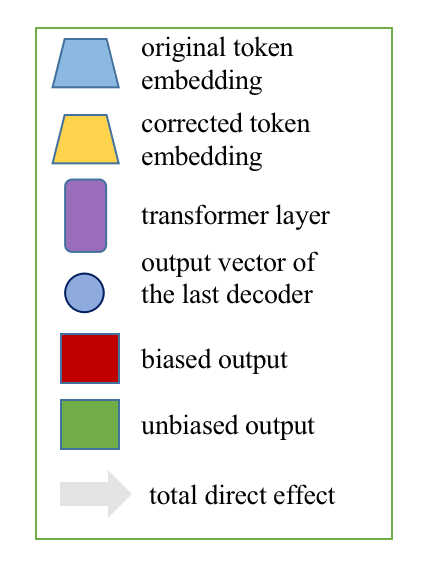}
         %\caption{}
         %\label{fig:pos_acc_samsum}
     \end{minipage}
     \begin{minipage}{0.8\textwidth}
         \centering
         \includegraphics[width=0.9\textwidth]{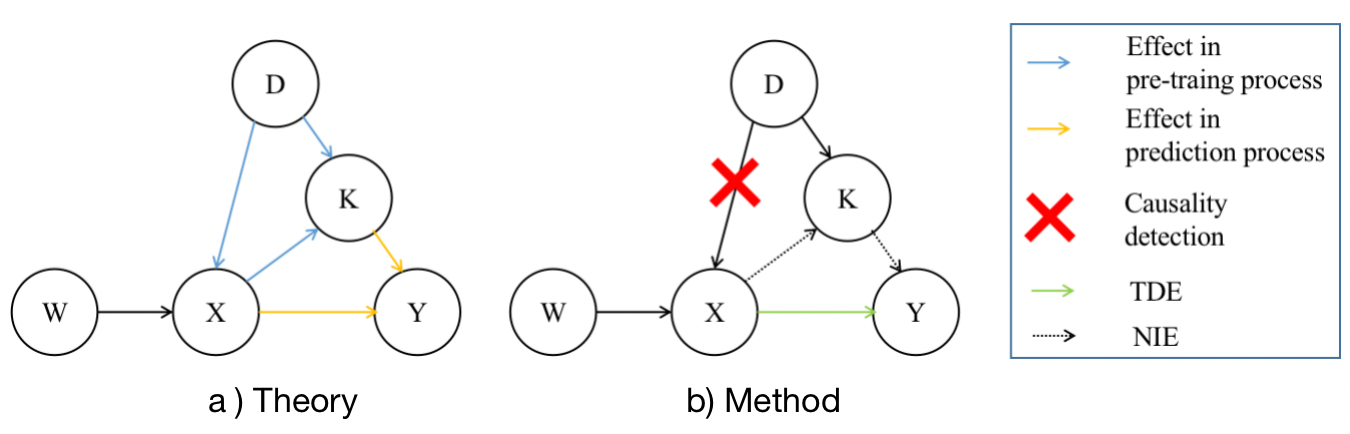}
         %\caption{}
         %\label{fig:pos_acc_samsum}
     \end{minipage}
        \caption{A directed acyclic graph used to indicate how variables of pre-training data (D), occupation word (W), input feature (X), pre-trained knowledge (K), and gender prediction (Y) interact with each other through causal links.}
        \label{fig:causal}
        \vskip -0.2in
\end{figure*}

Literature has revealed that the gender bias of the language models is oriented from the imbalance of the training data (\cite{bordia2019identifying}). Figure \ref{fig:causal} i explain the pathways of how the training-data imbalance misleads the embedding and transformer layers and consequently causes gender bias. We theoretically clarify the causal mechanism of how the data imbalance causes model gender biases (Figure). Analyzing the causal diagram enables us to have two discoveries: 1) there exists a pathway, the gender bias oriented by which can be fully mitigated while maintaining the model performance. 2) mitigating gender bias can be achieved by an explainable and controllable fine-tuning approach rather than augmentation. 

We discover that the gender bias can be fully mitigated without hurting the model-performance degradation if the bias caused by that the training-data imbalance misleading the fitting of the embedding layers. In the causal diagram Figure \ref{fig:causal} a, the pathway from $D \to X \to Y$ denotes the mechanism. The $D \to X$ denotes the mechanism that data imbalance causes incorrect embedding during training while $X \to Y$ denotes the dynamics that the incorrect $X$ causes the model's gender-biased output. In the misled language model, the gender-mutual words are incorrectly mapped to the embedding vectors that are not perpendicular to the gender subspace \cite{zhaoetal2018learning,bolukbasi2016man}. When the pre-trained model is applied to predict $Y$ according to $W$, the model first converts $W$ to the incorrect embedding vector $X$ and thus leads to a gender-biased prediction $\hat{Y}$. For example, if the word "doctor" is more frequently associated with male words in the training data, it will be mapped to a gender-biased embedding vector. Consequently, gender-bias errors will occur when the words "doctor" is input into the pre-trained model for prediction. 

If we are able to correct the parameters of the embedding layer only, we can mitigate the embedding-error-induced model gender bias while avoiding the model-performance degradation according to the literature. However, in the prediction process, the output $Y$ is simultaneously determined by embedding $X$ and transformer $K$. The data imbalance can also mislead the parameter fitting of the transformer layers $K$. If we pursue the mitigation of the embedding-error-induced model, we have to disentangle the effect induced by $X$'s bias from the effect of $K$'s bias.

Analyzing the causal diagram Fig.~\ref{fig:causal} confirms that the effects of the biases of $X$ and $K$ can be disentangled. When a gender-mutual word $W$ is inputted to predict $Y_0$, a biased model $M$ will output a gender-biased prediction $\hat{Y}_M$. Here, $Y_0$ refers to the probability that the output word is $male$ when there is no gender bias and $Y_0=0.5$. We use total effect ($TE$) to denote the size of the gender bias caused by both the errors in $X$ and $K$
\begin{equation}
\begin{aligned}
    TE &=|\mathbb{E}[\hat{Y}_M|W=w]-\mathbb{E}[Y_0|W=w]|\\
       &=\mathbb{E}[Y|X=x,K=k]-\mathbb{E}[Y|X=x_0,K=k_0],
\end{aligned}
\end{equation}
where $x,k$ are the current value of parameter $X,K$ and $x_0,k_0$ are the correct value of the parameters.
Note that both $X$ and $K$ are the parts of the model $M$ and thus determine TE. In the following theorem, we rigorously explain and prove the two gender biases can be disentangled.
\begin{theorem}\label{lem:decomposed}
The incorrect $X$ misled during the training process is able to be adjusted alone without affecting $K$.
\end{theorem}
\vskip -0.1in
\begin{proof}
The total effect can be decomposed into the sum of two parts, the Total Direct Effect (TDE;\citealp{pearl2013direct}) and the Natural Indirect Effect (NIE;\citealp{pearl2013direct}), as follow
\begin{equation}
 \begin{aligned}
      \left. \begin{array}{cc}
    TE=&(\mathbb{E}[Y|X=x,K=k]       \\
       &-\mathbb{E}[Y|X=x_0,K=k]) \\
    \end{array}
     \right\} & TDE\\
     \left. \begin{array}{c}
    +(\mathbb{E}\left[Y|X=x_0,K=k\right]\\
    -\mathbb{E}\left[Y|X=x_0,K=k_0\right]).
    \end{array}
     \right\} & NIE\\
\end{aligned}
\end{equation}
\vskip -0.4in
\end{proof}
According to the above theorem, the TDE is purely caused by gender-biased embedding. Further, the size of TDE is monotonic to the gap between the gender-biased embedding vectors and the correct one. The $X$ minimizing TDE is also the embedding vectors avoiding the gender bias. Thus, if we are able to figure out the $X$ that minimizes TDE, we only correct the embedding-error-induced gender bias and will not degrade the model performance. 

\textbf{Remark:} The above theoretic analysis also implies why the current fine-tuning methods of debiasing for the pre-trained model in literature cause performance degradation. We argue that the current literature aims at minimizing TE rather than TDE, which causes the risk of model-performance degradation. The current fine-tuning approach mainly modifies the objective function of the training process by including the gender-equality regularizer. Then, according to the modified objective function, the whole model is fine-tuned. While the whole TE is included in the objective function, the parameters of K can change during the fine-tuning process, which can cause a change in the model performance. While the literature has confirmed that adjusting the embedding-error-induced bias will not cause performance degradation. The degradation in literature has to come from adjusting the transformer-error-induced bias.

\vskip -0.1in
\section{Causality-Detection Debiasing for Double Dividend}

According to the theoretical analysis, we discover that the embedding-error-induced gender bias is able to be mitigated by a fine tuning approach. The above analysis implies that we shall explore a fine-tuning approach that only minimize TDE but not influence the parameters in the latter transformer layers. The fine-tuning approach enables us to only caliberate the parameters in the embedding layer while maintain all other parameters in the transformer layers. Here, we propose a {\bf D}ebiasing {\bf A}ppproach with {\bf M}aintaining {\bf P}erformance (DAMP) and apply it in GPT-2 model. The key of DAMP is to find the correct token embedding $X$ through TDE without any other impact, as shown in Figure \ref{fig:method}. So as to complete the task of correcting deviation without affecting the performance of the model. 

\begin{figure}[h!]
\centerline{\includegraphics[width=0.27\textwidth]{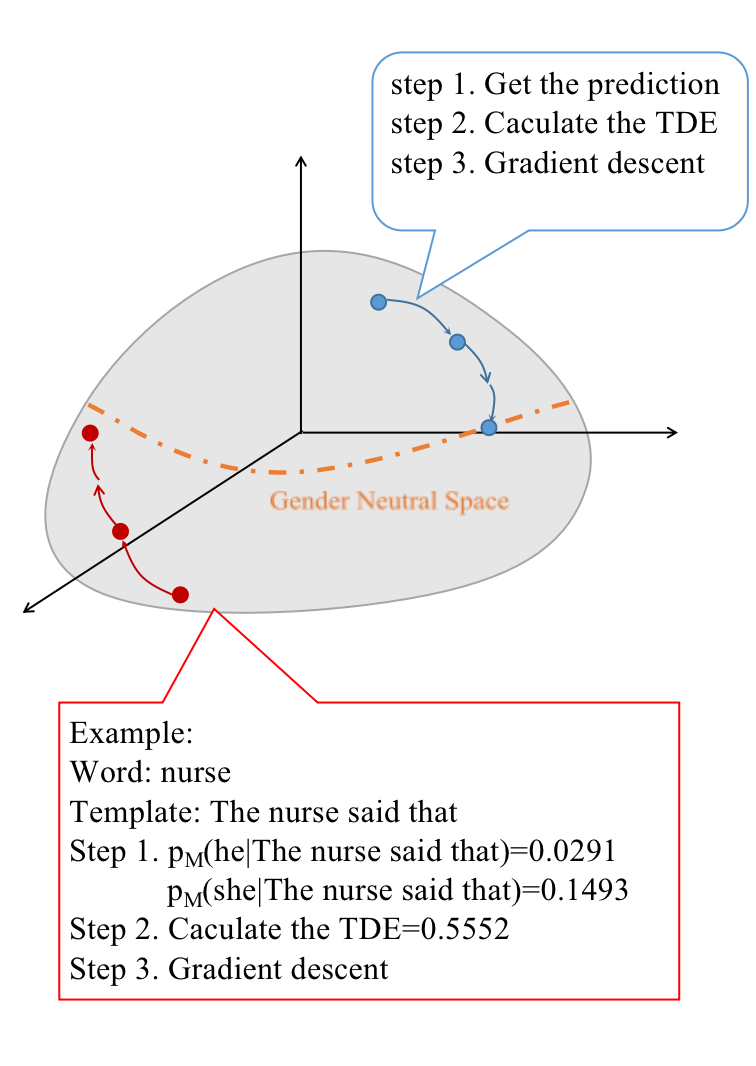}}
\caption{The panorama of our DAMP debiasing method with an example.}
\label{fig:method}
\vskip -0.3in
\end{figure}

There are three challenges of developing the fine-tuning approach for mitigate the embedding-error-induced gender bias. First, there lacks an systematical approach to systematically sample sufficient $\hat{Y}_M$, which is critical for calculating TDE. To estimate the TDE, it is necessary to generate templates to generate sufficient samples of $\hat{Y}_M$. In contrast, literature mostly focus on the performance in several particular templates. The sampling method is absent from the previous studies. Second, there is currently no established method to disentangle TDE from TE. Third, it is hard to directly reduce TDE through projection or other geometric approach because that the word embedding spaces of GPT-2 and other transformer language models are non-Euclidean. Thus, our DAMP method respectively solve the three challenges. Sequentially, the DAMP method includes three steps:
%Therefore, we have to figure out the method of 
%\begin{itemize}
%    \item Calculating the $TDE$ of $X$ for every gender-neutral word $W$;
%    \item Minimizing the $TDE$ to correcting $X$ towards %the ground truth $X_0$.
%\end{itemize}
%对TDE的分析展现给我们一个让预训练模型通过fine-tuning方法自己学会W→X真实因果关系、去除因为D带来的相关性噪音的方案。对于一个大规模预训练模型而言,TDE是衡量W→X关系被D扭曲程度的指标：TDE越低,那么W→X被D扭曲的水平就约小。因此,我们可以通过fine-tuning,让机器explore新的X来降低TDE,从而摈除Data 不平衡对机器学习W→X因果关系的干扰。
%In order to verify our theory, we implemented causal detection debiasing for double divide on GPT-2, which corrected the gender discrimination of occupation vocabulary  $\mathcal{V}_{o}$ without affecting the performance of the model. 
%下面这段不清楚。可以调整为,在GPT-2上,我们可以怎么算NDE。
%Each of the above tasks encounters a challenge. There lacks a method of estimating the $TDE$. Furthermore, the value of $X_0$ is unclear. Fortunately, the feature $X$ is stored in the token-embedding in the model of GPT-2, which consists of a continuous space \cite{bolukbasi2016man}. Therefore, we can develop a gradient-searching method of guiding $X$ to approach the unknown $X_0$. 

\textbf{Step 1.} Develop the approach of generating the templates for sufficiently sampling $\hat{Y}$%Get the language model's gender prediction $Y$ of a neutral word $W$ by the generated templates.

\textbf{Step 2.} Disentangle the TDE from TE by preventing the fine-tuning process from affecting $K$, the parameters of transformer layers. %which current embedding of professional words on gender justice.  We take the TDE, instead of the projection on the bias direction, as the distance between the embedding and the gender-neutral surface. 

\textbf{Step 3.} Develop the gradient approach to minimize the TDE value in order to correct the parameters of the embedding effects, so that the token embedding of input word is gender unbiased. 

In summary, the above three steps together enable a self-training fine-tuning approach to correct the errors in the embedding layer, which cause the language model's gender bias and are oriented by the training-data imbalanced. In the rest of this section, we provide the details of each step. 

%tasks manipulate the token embedding of the inputted word to reduce TDE. 
%we can enable a debias method by manipulating the token embedding of inputted word. We can correct the bias by manipulating the token embeddings from the gender-biased position to a position on the gender-neutral surface. The whole debiasing process can be divided into three steps.

%The value of $TDE$ can be estimated by 
%\citet{bolukbasi2016man} suggest that gender bias is captured by a bias direction in the word embedding. We discover the similar property in GPT-2. The gender bias occurs once the projection of a word's token embedding deviates from gender-neural surface. The deviation direction of token-embendding projection is parallel with the gender-bias direction. Thus, we suggest that the gender bias of GPT-2 model is contributed by the deviation of token embeddings of inputted words from the gender-neutral surface.

%We discover that gender bias occurs once the projection of a word's token embedding deviates from gender-neural hyper surface. The deviation direction of token-embendding projection is parallel with the gender-bias direction. Thus, we suggest that the gender bias of GPT-2 model is contributed by the deviation of token embeddings of inputted words from the gender-neutral hyper surface.

%\begin{figure*}[ht]
%\vskip -0.1in
%\begin{center}
%\centerline{\includegraphics[width=0.8\textwidth]{debiasing.pdf}}
%\caption{The panorama of our DAMP debiasing method with an example.}
%\label{fig:nespace}
%\end{center}
%\vskip -0.2in
%\end{figure*}

\subsection{Get Gender Prediction}\label{sec:gt}

Given a language model $M$ and a sequence of tokens $w_1, \cdots, w_k$, we define $p_M(w_{k+1} | w_1, ..., w_k)$ as the probability assigned by the language model to the token $w_{k+1}$ being the next token in the sequence. In general, the language model cannot directly provide the gender judgment of a certain professional vocabulary $w$. However, by designing a reasonable template as input for the model, we can get the model's prediction of the next word, and indirectly obtain the model's gender prediction $\hat{Y}$ for $w$. For example, when the template is "The nurse said that" and the options are the pronouns "he" and "she", we can observe the gender prediction $Y$ of "nurse" from the word selection of the model. We can calculate the distribution of $\hat{Y}$ using the following equation:

%Given a language model $M$ and a sequence of tokens $w_1, \cdots, w_k$, let $p_M(w_{k+1} | w_1, ..., w_k)$ denote the probability that the language model assigns to $w_{k+1}$ being the next token. For each token embedding $X$ of word $w$, we use natural language prompts to obtain the gender prediction $Y$ of $w$. That is, we supplement the generated text $t$ with options for the next word corresponding to the possible value of $Y$ and prompt the model to generate an answer. There is an example shown in Figure \ref{fig:nespace}. In the above example, the template is \textit{The nurse said that} and options is pronoun \textit{he,she} for the given word \textit{nurse}. Then we can calculate the distribution of Y by the following equ\ref{equ:distribution}.

\begin{equation*}\label{equ:distribution}
    \begin{array}{cc}  
             p(Y=\text{male}|t)= & \frac{p_M(\text{he}|t)}{p_M(\text{he}|t)+p_M(\text{she}|t)},\\
             p(Y=\text{female}|t)= & \frac{p_M(\text{she}|t)}{p_M(\text{he}|t)+p_M(\text{she}|t)}.
             \end{array}  
\end{equation*}

%\todo{discussion and illustration of the type of the generated templates}
Here, we propose an automatically template generating approach enabling the gradient search of our DAMP method. Our templates, denoted by $t$, are designed to meet two requirements:
\begin{itemize}
    \item \textbf{Intermediary:} The context should link the occupation and gender. To achieve this, we require the text to contain the target occupation words and have a probability of the pronoun "he" and "she" greater than some threshold $s$.
    \item \textbf{Neutrality:} The text should not contain information suggestive of gender. To achieve this, we require the text to be composed of gender-neutral words, specifically, we require $t$ to have no intersection with a list of gendered words $\mathcal{V}_{gender}$ provided by \citet{zhaoetal2018learning}.
\end{itemize}
Further, we propose a template generation algorithm that utilizes the text generation capabilities of GPT-2. Our algorithm starts by fixing the beginning of the template to include the target occupation word, and then uses GPT-2 to continue generating the template through top-k sampling. The sampling continues until the predicted probability of selecting "he" and "she" as the next word is greater than the threshold. To ensure efficiency, we set a maximum length for the top-k sampling at 15 words. If this maximum length is reached without satisfying the probability threshold, the sampling process is restarted from the beginning. Algorithm \ref{alg:tg} explains the detailed workflow of the template generation process.

\begin{algorithm}
   \caption{Templates Genderation}
   \label{alg:tg}
\begin{algorithmic}
   \STATE {\bfseries Input:} Occupation word $w$, GPT-2 model $p_M$ 
   \STATE {\bfseries Output:} A template for gender prediction 
   \STATE Initialize $t=$'the w'
   \REPEAT
   \STATE Get the next words $w'\sim p_M(w'|t)$ 
   \IF{$w'\notin {V}_{gender}$}
   \STATE Add $w'$ at the end of $t$
   \ENDIF
   \IF{$len(t)>15$}
   \STATE Set $t=$'the w'
   \ENDIF
   \UNTIL{ $p_M(he|t)>s,p_M(she|t)>s$}
   \STATE {\bfseries return:} t
\end{algorithmic}
\end{algorithm}
\vskip -0.2in
\subsection{Calculation of TDE}

With gender predicting the distribution of Y, we can calculate the  $TDE$  of $X$ on $Y$.
 For gender-neutral occupations, we know prior that they are independent of the gender prediction. To simplify the process, we use a uniform distribution to substitute for  $\hat{Y}_ k|do(X=x_0)$. Combined the Eq.\ref{equ:distribution}, we have
$$\left\{  
             \begin{array}{cc}  
             |TDE_t| = & |1/2-p_1|=|1/2-p_2|,  \\  
             p_1(t)= &\frac{p_M(\text{he}|t)}{p_M(\text{he}|t)+p_M(\text{she}|t)},\\  
             p_2(t)= &\frac{p_M(\text{she}|t)}{p_M(\text{he}|t)+p_M(\text{she}|t)}.    
             \end{array}  
\right.$$ 
for given template t. To avoid the influence of the template itself on the results, we select enough templates $t_1,...,t_n$ and use the average TDE for every template as the final TDE:
\begin{equation}
    TDE=\frac{1}{n}\sum_{i=1}^{n}TDE_{t_i}.
\end{equation}

%$$TDE=1+\frac{p_M(\text{he}|t)}{p_M(\text{he}|t)+p_M(\text{she}|t)}\log \frac{p_M(\text{he}|t)}{p_M(\text{he}|t)+p_M(\text{she}|t)} +\frac{p_M(\text{she}|t)}{p_M(\text{he}|t)+p_M(\text{she}|t)}\log \frac{p_M(\text{she}|t)}{p_M(\text{he}|t)+p_M(\text{she}|t)}$$

\subsection{Gradient Descent}

According to our previous discussion, as a debiasing method with minimal side effects, we should adjust the embeddings for occupational words in order to minimize the TDE, that is
\begin{equation} \label{eq:opt}
    \min_{do(X=\hat{x})} |TDE|.
\end{equation}
Due to the limitations of pre-trained knowledge, the variable $\hat{x}$ in optimization problem \ref{eq:opt} has a limited range of values. As noted by \citet{ethayarajh-2019-contextual}, word representations tend to occupy a narrow cone in the vector space. Therefore, it is important to ensure that the selected features meet the constraint
\begin{equation}
    ||\hat{x}-x||^2<r,
\end{equation}
where $x$ is the original word embedding and $r$ is the radius of the cone. Thus, the optimization problem is summarized as 
\begin{equation}
    \min_{do(X=\hat{x})} TDE \ s.t. \ ||\hat{x}-x||^2<r
\end{equation}
Due to the uncertainty of the specific value of r, we utilize the Sequential Unconstrained Minimization Technique (SUMT) to find a solution to the problem at hand. To do this, we design the following loss function:
\begin{equation}
 \begin{aligned}
    L= &\sum_{i}^{n}p_1(t_i)\log p_1(t_i)+p_2(t_i)\log p_2(t_i)\\
    &+\alpha||\hat{x}-x||^2,
\end{aligned}
\end{equation}
Next, we use an optimizer based on gradient descent to minimize the loss function $L$. In summary, we use algorithm \ref{alg:1} to obtain unbiased embeddings for gender-neutral occupation words $w$.

\begin{algorithm}[!h]
   \caption{DAMP Debiasing Method}
   \label{alg:1}
\begin{algorithmic}
   \STATE {\bfseries Input:} word $w$ and pre-trained LM $p_M$
   \STATE {\bfseries Output:} unbiased embedding $\hat{x}$ for $w$
   \STATE Extract the pre-trained embedding $\hat{x}_0=x$ of $w$.
   \STATE Generate the template $t_1,\cdots,t_n$.
   \FOR{$k=1$ {\bfseries to} $m$}
   \STATE Initialize $TDE =0$
   \FOR{$i=1$ {\bfseries to} $n$}
   \STATE $p_1= \frac{p_M(\text{he}|t)}{p_M(\text{he}|t)+p_M(\text{she}|t)}$
   \STATE $p_2= \frac{p_M(\text{she}|t)}{p_M(\text{he}|t)+p_M(\text{she}|t)}$
   \STATE $L_{TDE} = L_{TDE} + \frac{1}{n}(1+p_1\log p_1+p_2\log p_2)$
   \ENDFOR
   \STATE $L_{k-1}=L_{TDE}+\alpha||e_{k-1}-e_0||^2$
   \STATE $\hat{x}_k=\hat{x}_{k-1}-\lambda \frac{\partial L_{k-1}}{\partial e_{k-1}}$
   \STATE Set $\hat{x}_k$ as the embedding of $w$ in $p_M$.
   \ENDFOR
   \STATE \textbf{return} $\hat{x}=\hat{x}_m$
\end{algorithmic}
\end{algorithm}
%However, due to the complexity of GPT-2 model, it is difficult to solve the optimal solution of problem \ref{equ:ob1}. Thus we obtain an approximate optimal solution by reducing the loss function

\begin{table*}[!t]
\caption{The perplexity results of the original and different debiased GPT-2 models.  If the perplexity score is lower, the model performs better. 
We show the lowest results of perplexity in bold.}
\label{tab:perplexity}
    \centering
    \begin{tabular}{|l|l|l|l|l|l|l|}
    \hline
        perplexity & vanilla &  INLP & SD & GEL & DAMP(ours) \\ \hline
        small & 25.1711 &  6.73E+16 & 2.67E+11 & 1.09E+46 & \textbf{26.2769} \\ \hline
        medium & 18.4633 &  1.68E+15 & 157.6118 & 2.05E+32 & \textbf{36.9222}\\ \hline
        large & 16.4447  & 21.0505 & \textbf{16.4732} & 17.133 & 16.5336 \\ \hline
        xl & 14.7881 & 17.6801 & 15.2138 & 15.451 & \textbf{14.9612} \\ \hline
    \end{tabular}
\end{table*}
\vskip -0.2in
\begin{table*}[]
\caption{Scores of different debiasing methods on 9 test tasks of GLUE. We bold the result with the highest average result.}
\label{tab:glue}
\centering
\begin{tabular}{lllllllllll}
\hline
     & CoLA  & MNLI  & MRPC  & QNLI  & QQP   & RTE   & SST   & STS-B & WNLI  &   Average       \\
\hline
GPT-2(small) & 29.1  & 82.43 & 84.51 & 87.71 & 89.18 & 64.74 & 91.97 & 84.26 & 43.19 & 73.01    \\
INLP & 31.79 & 82.73 & 84.34 & 87.81 & 89.17 & 64.38 & 92.01 & 83.99 & 41.31 & 73.06    \\
SD   & 30.2  & 82.56 & 84.43 & 87.9  & 89.09 & 64.86 & 91.97 & 84.18 & 38.5  & 72.63    \\
GEL  & 16.33 & 82.21 & 86.42 & 87.81 & 85.02 & 66.43 & 92.32 & 84.09 & 35.21 & 70.65 \\
DAMP  & 31.87 & 82.62 & 85.48 & 87.81 & 89.22 & 64.98 & 92.55 & 82.77 & 40.89 & \textbf{73.13}\\
\hline
GPT-2(medium) & 52.43 & 85.92 & 83.26 & 87.71 & 87.87 & 64.94 & 93.58 & 87.38 & 49.3  & 76.93 \\
INLP & 50.46 & 85.66 & 87.33 & 90.83 & 87.78 & 67.15 & 93.81 & 87.31 & 28.17 & 75.39 \\
SD   & 53.12 & 85.63 & 88.25 & 90.46 & 87.84 & 64.98 & 93.92 & 87.22 & 39.44 & 76.76 \\
GEL  & 47.85 & 85.6  & 86.66 & 87.81 & 87.79 & 66.79 & 94.27 & 86.54 & 40.85 & 76.01 \\
DAMP  & 51.54 & 85.94 & 82.54 & 90.72 & 87.92 & 63.9  & 94.5  & 87.44 & 46.48 & \textbf{76.78}\\
\hline
GPT-2(large) & 55.44 & 85.32 & 88.78 & 91.14 & 88.09 & 74.73 & 89.31 & 89.71 & 43.66 & 78.46 \\
INLP & 60.07 & 84.26 & 90.2  & 92.13 & 84.26 & 75.45 & 93.92 & 89.49 & 33.8  & 78.18 \\
SD   & 58.8  & 83.99 & 89.04 & 91.87 & 85.31 & 74.01 & 94.04 & 89.8  & 33.8  & 77.85 \\
GEL  & 58.65 & 84.48 & 88.08 & 91.82 & 85.19 & 74.73 & 89.98 & 89.98 & 25.35 & 76.47 \\
DAMP  & 53.16 & 85.23 & 90.24 & 91.76 & 87.92 & 74.37 & 92.78 & 89.31 & 43.66 & \textbf{78.71}\\
\hline
\end{tabular}
\end{table*}

%\begin{algorithm}
%\caption{DAMP method on one word} %算法的名字
%\hspace*{0.02in} {\bf Input:} %算法的输入, \hspace*{0.02in}用来控制位置,同时利用 \\ 进行换行
%pre-trained LM $p_\theta^{*}$ and word list $\mathcal{V}$\\
%\hspace*{0.02in} {\bf Output:} %算法的结果输出
%debiased LM $p_\theta$
%\begin{algorithmic}[1]
%\For{$v\in \mathcal{V}$}
%\State set $p_\theta=p_\theta^{*}$
%\State generate templates $t_1,\cdots,t_n$
%\State Get the original token embedding $e_0$ of $v$
%\For{$k=1,2,\cdots,m$}
%\State $L_{k-1}=\alpha||e_{k-1}-e_0||^2+b(e_{k-1})$
%\State $e_k=e_{k-1}-\lambda \frac{\partial L_{k-1}}{\partial e_{k-1}}$
%\EndFor
%\State update $e_m$ as the token embedding of $p_\theta$.
%\EndFor
%\State \Return $p_\theta$
%\end{algorithmic}\label{alg:1}
%\end{algorithm}

For the overall model, we use Algorithm 1 to reduce gender bias for each gender-neutral occupation in the occupation vocabulary. We then replace the embeddings of occupation words with the debiased embeddings, to create a bias-free language model.

In practice, the token representations of different occupation words may share the same tokens due to Byte Pair Encoding (BPE). For example, in GPT-2, the token representation of "jeweler" is $[16927, 263]$ and the token representation of "entertainer" is $[8204, 263]$. They both contain the same token $[263]$. In this scenario, we use the average value of the debiased embeddings corresponding to the token as the final token embedding.
%The abridged general view is shown in the figure in appendix.

\section{Experiments}

\subsection{Setup}
\paragraph{Model} To evaluate the effectiveness of our $DAMP$ method in reducing gender biases, we conduct experiments on the GPT-2 model trained in English, using various model sizes: small (117M parameters), medium (345M parameters), large (774M parameters), and extra-large (xl, 1558M parameters) \cite{radford2019language}. The pre-trained weights are obtained from the Transformers Python library \cite{2019huggingFace}.

\paragraph{Templates} Our templates are generated using the method outlined in Section \ref{sec:gt}. For each occupation word, we generate 500 templates for debiasing. The threshold s is set to $0.08$. We use the occupation vocabulary provided by \citet{vig2020investigating}.

\paragraph{Hyper Parameters} In Algorithm \ref{alg:1}, we set the number of templates $n$ to $500$, the number of optimization iterations $m$ to $100$, and the hyper-parameter $\alpha$ to $1000$. In the experiment, we use the ADAM algorithm to minimize the loss function and set the learning rate $\lambda$ to $0.002$.

\begin{table*}\centering
\caption{The test results on StereoSet of different debiased GPT-2 models. Language Modeling Score (lms) is the percentage of the meaningful answer that the language model prefers over the meaningless association, which is the higher the better. Stereotype Score (ss) is the percentage of the stereotypical association that the model prefers over the anti-stereotypical association, which is closer to 50 the better. Idealized CAT Score (icat), a metric for comprehensive evaluation of bias and performance, is calculated by $lms*\min(ss,10-ss)/50$, which is the higher the better. We show the best results of lms, ss, and icat in bold.}
\label{tab:stero}
\begin{tabular}{|l|lllll|lllll|}
\hline
     & \multicolumn{5}{c|}{small}       & \multicolumn{5}{c|}{medium} \\
     \hline
     & vanilla & INLP   & SD    & GEL   & DAMP  & vanilla  & INLP   & SD     & GEL   & DAMP   \\
     \hline
lms  & 89.47 & 40.98  & 70.18 & 59.6  & \textbf{89.47} & 71.93 & 56.14  & 70.18  & 58.02 & \textbf{71.93} \\
\hline
ss   & 55.77 & 52     & 52.5  & 54.24 & \textbf{49.02} & 53.66 & 53.13  & 52.5   & \textbf{51.06} & 46.34 \\
\hline
icat & 80.70 & 39.34  & 66.67 & 54.55 & \textbf{87.72} & 66.67 & 52.63  & \textbf{66.67}  & 56.79 & \textbf{66.67} \\
\hline
     & \multicolumn{5}{c|}{large} & \multicolumn{5}{c|}{xl}     \\
     \hline
     & vanilla & INLP   & SD    & GEL   & DAMP  & vanilla & INLP   & SD     & GEL   & DAMP   \\
     \hline
lms  & 84.21 & 82.45  & \textbf{87.72} & 78.95 & 84.21 & 85.96 & 82.46  & \textbf{89.47}  & 78.95 & 85.96 \\
\hline
ss  &  52.08 & 55.32  & 54    & 44.44 & \textbf{50}  & 55.10  & 53.2   & 54.9   & 44.44 & \textbf{51.02} \\
\hline
icat & 80.70 & 73.68  & 80.7  & 70.18 & \textbf{84.21} & 77.19 & 77.19  & 80.7   & 70.18 & \textbf{84.21}\\
\hline
\end{tabular}
\vskip -0.2in
\end{table*}

\iffalse
\begin{table*}[t]
\caption{}
    \centering
    \begin{tabular}{|l|l|l|l|l|}
    \hline
        ~  & INLP & Sentence Debias & Gender-Equalizing Loss & our  \\ \hline
        gpt-2   & 0.0826  & 0.9991  & 0.1759  & 0.1437   \\ \hline
        gpt2-medium  & 0.6240  & 0.2984  & 0.0720  & 0.1163   \\ \hline
        gpt2-large   & 0.2192  & 0.3070  & 0.1204  & 0.0847   \\ \hline
        gpt2-xl  & 0.2538  & 0.3522  & 0.0950  & 0.0783  \\ \hline
    \end{tabular}
\vskip -0.1in
\end{table*}
\fi

%Next, we are comparing bias and perplexity with baseline. Take the result in GPT-2(medium) as example, which is shown in table \ref{tab:med}, our self training method not only significantly reduce the bias, and also controls the increase of confusion. The debias results of GPT-2 with other scale are in the appendix.
\begin{table*}[!h]
\caption{The bias effect size results of different debiasing methods on the test sets (SEAT-6 SEAT-6b ). If the result is closer to zero, the model has a lower level of bias.}
\label{tab:seat}
%\vspace{-0.15in}
\begin{center}
%\begin{small}
%\begin{sc}
\begin{tabular}{|l|llll|llll|}
\hline
            & \multicolumn{4}{c|}{small}                                                                                                                                                                     & \multicolumn{4}{c|}{medium}                                                                                                                                                              \\ \hline
            & \multicolumn{1}{l|}{INLP}                          & \multicolumn{1}{l|}{SD}                            & \multicolumn{1}{l|}{GEL}                            & DAMP                           & \multicolumn{1}{l|}{INLP}                           & \multicolumn{1}{l|}{SD}                            & \multicolumn{1}{l|}{GEL}                           & DAMP                           \\ \hline
SEAT-6      & \multicolumn{1}{l|}{0.2398}                        & \multicolumn{1}{l|}{{ 0.2026}} & \multicolumn{1}{l|}{0.3892}                         & { 0.1377} & \multicolumn{1}{l|}{{ -0.1671}} & \multicolumn{1}{l|}{0.8271}                        & \multicolumn{1}{l|}{{ 0.1589}} & 0.2678                        \\ \hline
SEAT-6b     & \multicolumn{1}{l|}{0.8656}                        & \multicolumn{1}{l|}{{ 0.0097}} & \multicolumn{1}{l|}{0.3033}                         & { 0.0029} & \multicolumn{1}{l|}{0.8158}                         & \multicolumn{1}{l|}{{ 0.1181}} & \multicolumn{1}{l|}{-0.1182}                       & { 0.0708} \\ \hline
            & \multicolumn{4}{c|}{large}                                                                                                                                                               & \multicolumn{4}{c|}{xl}                                                                                                                                                                  \\ \hline
            & \multicolumn{1}{l|}{INLP}                          & \multicolumn{1}{l|}{SD}                            & \multicolumn{1}{l|}{GEL}                            & DAMP                           & \multicolumn{1}{l|}{INLP}                           & \multicolumn{1}{l|}{SD}                            & \multicolumn{1}{l|}{GEL}                           & DAMP                           \\ \hline
SEAT-6      & \multicolumn{1}{l|}{1.2298}                        & \multicolumn{1}{l|}{1.1103}                        & \multicolumn{1}{l|}{{ 0.8263}}  & {1.0937} & \multicolumn{1}{l|}{1.0574}                         & \multicolumn{1}{l|}{1.0532}                        & \multicolumn{1}{l|}{{ 0.8263}} & {1.0335} \\ \hline
SEAT-6b     & \multicolumn{1}{l|}{0.3269}                        & \multicolumn{1}{l|}{0.2491}                        & \multicolumn{1}{l|}{{-0.0599}} & {0.1933} & \multicolumn{1}{l|}{{0.2539}}  & \multicolumn{1}{l|}{0.282}                         & \multicolumn{1}{l|}{{ -0.128}} & 0.304                         \\ \hline
\end{tabular}
%\end{sc}
%\end{small}
\end{center}
\vskip -0.2in
\end{table*}

\subsection{Baseline}

We compare our $DAMP$ method with the following methods: Iterative Nullspace Projection (INLP; \citealp{ravfogel2020null}), Sentence Debias(SD; \citealp{liangetal2020towards}), and fine-tuning method with Gender Equalizing Loss (GEL; \citealp{qian-etal-2019-reducing}, \citealp{barikeri-etal-2021-redditbias}) \footnote{The implementation of the baseline methods uses code from  \href{https://github.com/umanlp/RedditBias}{https://github.com/umanlp/RedditBias}  and \href{https://github.com/umanlp/RedditBias}{https://github.com/umanlp/RedditBias}.}. 
Details of the baseline methods are provided in Appendix.

We apply our approach and the baselines to the pre-trained GPT-2 models of different sizes and compare them with several benchmarks and metrics. We compared the $DAMP$ method with the baseline methods in two aspects: the effectiveness of debiasing and the capability of performance maintenance. 
We compare the $DAMP$ method with the baseline methods in two aspects: the effectiveness of debiasing and the ability to maintain performance. We measure the effectiveness of debiasing using the metrics associated with StereoSet \cite{nadeem-etal-2021-stereoset} and SEAT (Sentence Encoder Association Test) \cite{may-etal-2019-measuring}. We evaluate the ability to maintain performance using the metric of Perplexity \cite{Merity2016Pointer} and the metrics proposed in StereoSet and GLUE (General Language Understanding Evaluation) \cite{wangetal2018glue}.
\vskip -0.05in
\textbf{Perplexity} Perplexity is a measure of how well a language model performs. A lower perplexity indicates better performance. We compute the perplexity on the wikitext-2 dataset \cite{Merity2016Pointer} for GPT-2 models of different sizes using four debiasing methods. The results are shown in Table~\ref{tab:perplexity}. 
\vskip -0.05in
\textbf{GLUE
%\footnote{The test code is from \href{https://github.com/jorgebastida/glue}{https://github.com/jorgebastida/glue}}
}
We also evaluated the performance of the debiased models on the GLUE benchmark, which is a multi-task natural language understanding benchmark and analysis platform developed by \citet{wangetal2018glue}. The test results on the GLUE dataset reflect the understanding ability of the language model. The performance of the debiased small GPT-2 is shown in Table~\ref{tab:glue}, and the results of medium GPT-2 and large GPT-2 are shown in the Appendix. 
\vskip -0.05in
\textbf{StereoSet
%\footnote{The data set is from \href{https://stereoset.mit.edu}{https://stereoset.mit.edu}. We only used test data related to both gender and occupation.}
}
StereoSet is a large-scale dataset for measuring stereotypical bias and performance of the language models\cite{nadeem-etal-2021-stereoset}. Associated with the dataset, \cite{nadeem-etal-2021-stereoset} also provided the metrics to assess the model's stereotypical bias and performance. We selected the part data in Stereoset that is related to both gender and occupation for testing, and the results are shown in the table\ref{tab:stero}.
\vskip -0.05in
\textbf{SEAT
%\footnote{The dataset is available from \href{https://aclanthology.org/attachments/N19-1063.Datasets.zip}{https://aclanthology.org/attachments/N19-1063.Datasets.zip}}
} Sentence Encoder Association Test (SEAT) is an extension by \citet{may-etal-2019-measuring} from the Word Embedding Association Test \cite{caliskan2017semantics}. SEAT is able to be adopted to assess the size of gender bias and has the associated metrics. We tested all the methods by SEAT and the result is shown in Table \ref{tab:seat}.

\subsection{Results}

\begin{figure*}[h!]
\centerline{\includegraphics[width=0.8\textwidth]{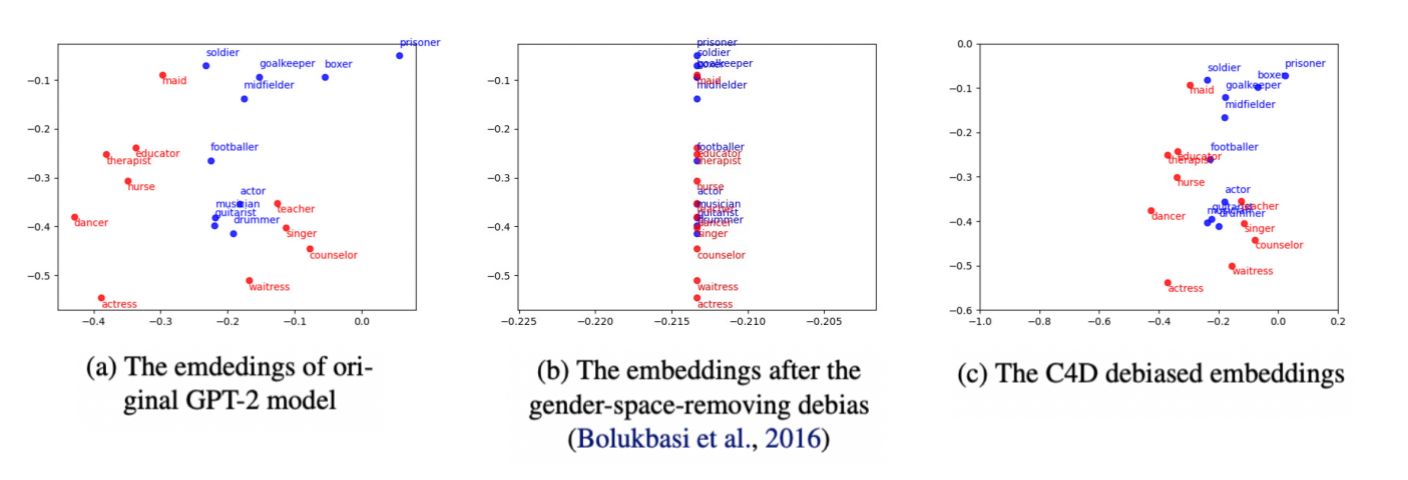}}
\caption{The token embeddings’  projection of selected words of occupations to the gender subspace. }
\end{figure*}

\begin{figure*}[h!]
\centerline{\includegraphics[width=0.8\textwidth]{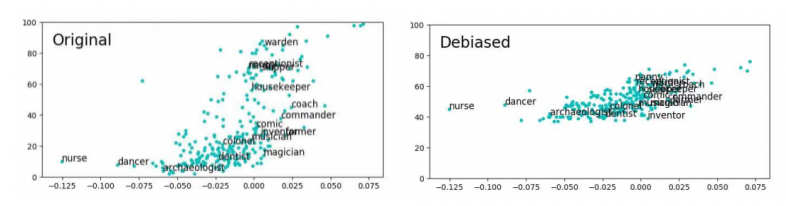}}
\caption{The number of male neighbors for each token embedding of occupation changes with its original bias, before and after debiasing. }
\label{fig:nei}
\end{figure*}
%We analyze our results from two aspects, performance and debiasing. In summary, our model is the only one simultaneously leads to significant mitigation of gender bias and maintains the model performance. 

The DAMP method effectively mitigates gender bias. In some cases, its debiasing effectiveness is equivalent to or even surpasses that of existing methods. For example, it has the lowest gender-bias score on the Stereoset for GPT-2 of small, large, and xl sizes. On the SEAT dataset, the DAMP method ranks first in the tests of small-seat6, small-seat6b, and medium-seat6b. 

Furthermore, the DAMP method effectively preserves model performance. Among all tested methods, it causes the lowest perplexity on the GPT-2 of small, medium, and xl sizes. On the large GPT-2, its perplexity ranks second, only $0.4\%$ worse than the highest score, while the highest-scoring method has significantly worse debiasing effectiveness. On the GLUE dataset, the DAMP method achieves the best average score. These results demonstrate that our method effectively preserves model performance while mitigating gender bias.

%The $DAMP$ method also essentially mitigates the gender bias. In some cases, the $DAMP$'s effectiveness of debiasing is equivalent to or even better than many of the current methods. For instance,  the $DAMP$ has the lowest gender-bias score on the Stereoset for GPT-2 of small, large, and xl sizes. On SEAT, the $DAMP$ method ranks first in the tests of small-seat6, small-seat6b, and medium-seat6b. 

%At the same time, the DAMP method can effectively protect the model performance. Among all tested methods, our $DAMP$ method causes the lowest perplexity on the GPT-2 of the small, medium, and xl sizes. The perplexity of $DAMP$ ranks second on the large GPT-2 but only $0.4\%$ worse than the highest score. However, the highest-score method has significantly worse debiasing effectiveness than $DAMP$. On the GLUE dataset, the $DAMP$ method achieves the best average score. These pieces of evidence show that our method can effectively protect the model performance.

Note it is not a surprise that the DAMP does not the most effective debiasing approach in some tests because DAMP only minimizes TDE rather than TE. In the tests that NIE significantly contributes the TE, the DAMP may have less debiasing effectiveness than those approaches minimizing TE. However, the tests manifest that DAMP has much better model performance than the TE-optimal methods in those tests. Further, the DAMP also substantially mitigated the gender bias in those tests that NIE has a large size. Consequently, the comprehensive scores of DAMP are always the best in all tests. 

Overall, our approach is able to correct the gender-occupation bias with little impact on performance. On all models of all sizes, the Stereoset comprehensive evaluation index $icat$ of our method exceeded the baseline. Although the correction effect of our method is not as good as that of some baseline methods at some test points, our method can better protect the model performance under these circumstances. This is consistent with our causal theory.

In summary, our analysis of the results revealed that our model effectively balances performance and debiasing. Specifically, it is the only model that simultaneously achieves a notable reduction in gender bias while maintaining high performance.
%In order to ensure the performance of the model, we only intervene in the path of partial bias that doesn't tangle with performance.

%On the one hand, from the results of perplexity (Table \ref{tab:perplexity}), our DAMP debiasing method has the best protection for the performance of the pre-trained model compared with other methods. Especially for GPT-2 models of large and xl , our method has hardly any performance degradation compared with the original model. On the other hand, in the test of SEAT, our method significantly reduces the bias level of the model. Especially for small and medium scale models, our correction method presents the state of the art performance. The projection of token embedding space in our method can explain the reason of double dividend, that is, our method keeps the original shape while making the projection on gender subspace of embeddings more concentrated.
\subsection{Geometry of Word Embedding}

 To analyze the mechanism of double dividends in our $DAMP$ method, we compare its functioning process with the removing-gender-subspace method of \citet{bolukbasi2016man} from a geometric perspective. Using Fig.~3, we project token embeddings of a group of words before and after debiasing onto a two-dimensional plane, which represents the gender subspace calculated by PCA. The left-hand side plot shows the distribution of token embeddings before debiasing, the center plot shows the distribution after the gender-space-removing debiasing, and the right-hand side plot shows the distribution after the DAMP debiasing. The gender-space-removing debiasing method forces the embeddings to be perpendicular to the $X$-axis, which can destroy other attributes of word embeddings outside the gender, leading to a decline in the performance of the model. In contrast, the DAMP debiasing method retains the topology of the original embeddings while reducing the distance between professions in the sex subspace. The retention of non-gender information in the topology structure is an important reason for the obvious retention of performance.

In addition, we can also observe the cleanliness of bias elimination through geometric methods. 
\citet{gonengoldberg2019lipstick} argues that the debiasing methods, such as HARD-DEBIASED \cite{barikeri-etal-2021-redditbias} and GN-GLOVE \cite{zhaoetal2018learning}, only remove bias by projection but do not handle bias by neighbors. \citet{gonengoldberg2019lipstick} proves the comment by the projection-neighbor figure which is almost no change before and after debiasing. For our method, we also draw the projection-neighbor figure in the same way, as shown in Figure~\ref{fig:nei}. It can be observed that the occupational words, are distributed near an S-curve before debiasing and near a straight line parallel to the x-axis after debiasing. This shows that our method can mitigate the bias-by-neighbors, which also explains why our method can effectively remove the bias of the language model. 
 
\section{Conclusions}

We propose a causal framework that explains the origin of gender bias and unifies the mechanisms for debiasing pre-trained models. This framework also addresses the bias-performance trade-off that is present in current debiasing methods. Based on this theory, we develop the DAMP method which reduces gender bias by adjusting token embedding. The experimental results demonstrate that our method can significantly reduce gender bias while maintaining performance.

%二象限图

\bibliographystyle{unsrtnat}
\bibliography{references} 

\begin{thebibliography}{32}
\providecommand{\natexlab}[1]{#1}
\providecommand{\url}[1]{\texttt{#1}}
\expandafter\ifx\csname urlstyle\endcsname\relax
  \providecommand{\doi}[1]{doi: #1}\else
  \providecommand{\doi}{doi: \begingroup \urlstyle{rm}\Url}\fi

\bibitem[Bolukbasi et~al.(2016)Bolukbasi, Chang, Zou, Saligrama, and
  Kalai]{bolukbasi2016man}
Tolga Bolukbasi, Kai-Wei Chang, James~Y Zou, Venkatesh Saligrama, and Adam~T
  Kalai.
\newblock Man is to computer programmer as woman is to homemaker? debiasing
  word embeddings.
\newblock \emph{Advances in neural information processing systems},
  29:\penalty0 4349--4357, 2016.

\bibitem[Zhao et~al.(2018)Zhao, Zhou, Li, Wang, and
  Chang]{zhaoetal2018learning}
Jieyu Zhao, Yichao Zhou, Zeyu Li, Wei Wang, and Kai-Wei Chang.
\newblock Learning gender-neutral word embeddings.
\newblock In \emph{Proceedings of the 2018 Conference on Empirical Methods in
  Natural Language Processing}, pages 4847--4853, Brussels, Belgium, 2018.
  Association for Computational Linguistics.
\newblock \doi{10.18653/v1/D18-1521}.
\newblock URL \url{https://aclanthology.org/D18-1521}.

\bibitem[Bhardwaj et~al.(2021)Bhardwaj, Majumder, and
  Poria]{bhardwaj2021investigating}
Rishabh Bhardwaj, Navonil Majumder, and Soujanya Poria.
\newblock Investigating gender bias in bert.
\newblock \emph{Cognitive Computation}, pages 1--11, 2021.

\bibitem[Vig et~al.(2020)Vig, Gehrmann, Belinkov, Qian, Nevo, Singer, and
  Shieber]{vig2020investigating}
Jesse Vig, Sebastian Gehrmann, Yonatan Belinkov, Sharon Qian, Daniel Nevo,
  Yaron Singer, and Stuart~M Shieber.
\newblock Investigating gender bias in language models using causal mediation
  analysis.
\newblock In \emph{NeurIPS}, 2020.

\bibitem[Sweeney(2013)]{sweeney2013discrimination}
Latanya Sweeney.
\newblock Discrimination in online ad delivery.
\newblock \emph{Communications of the ACM}, 56\penalty0 (5):\penalty0 44--54,
  2013.

\bibitem[Cowgill(2018)]{cowgill2018bias}
Bo~Cowgill.
\newblock Bias and productivity in humans and algorithms: Theory and evidence
  from resume screening.
\newblock \emph{Columbia Business School, Columbia University}, 29, 2018.

\bibitem[Deshpande et~al.(2020)Deshpande, Pan, and
  Foulds]{deshpande2020mitigating}
Ketki~V Deshpande, Shimei Pan, and James~R Foulds.
\newblock Mitigating demographic bias in ai-based resume filtering.
\newblock In \emph{Adjunct Publication of the 28th ACM Conference on User
  Modeling, Adaptation and Personalization}, pages 268--275, 2020.

\bibitem[Dressel and Farid(2018)]{dressel2018accuracy}
Julia Dressel and Hany Farid.
\newblock The accuracy, fairness, and limits of predicting recidivism.
\newblock \emph{Science advances}, 4\penalty0 (1):\penalty0 eaao5580, 2018.

\bibitem[Dastin(2018)]{Reuters2018Amazon}
J.~Dastin.
\newblock Amazon scraps secret ai recruiting tool that shows bias against
  women.
\newblock \emph{Reuters}, 2018.

\bibitem[Barikeri et~al.(2021)Barikeri, Lauscher, Vuli{\'c}, and
  Glava{\v{s}}]{barikeri-etal-2021-redditbias}
Soumya Barikeri, Anne Lauscher, Ivan Vuli{\'c}, and Goran Glava{\v{s}}.
\newblock {R}eddit{B}ias: A real-world resource for bias evaluation and
  debiasing of conversational language models.
\newblock In \emph{Proceedings of the 59th Annual Meeting of the Association
  for Computational Linguistics and the 11th International Joint Conference on
  Natural Language Processing (Volume 1: Long Papers)}, pages 1941--1955,
  Online, August 2021. Association for Computational Linguistics.
\newblock \doi{10.18653/v1/2021.acl-long.151}.
\newblock URL \url{https://aclanthology.org/2021.acl-long.151}.

\bibitem[Meade et~al.(2022)Meade, Poole-Dayan, and
  Reddy]{meadeetal2022empirical}
Nicholas Meade, Elinor Poole-Dayan, and Siva Reddy.
\newblock An empirical survey of the effectiveness of debiasing techniques for
  pre-trained language models.
\newblock In \emph{Proceedings of the 60th Annual Meeting of the Association
  for Computational Linguistics (Volume 1: Long Papers)}, pages 1878--1898,
  Dublin, Ireland, May 2022. Association for Computational Linguistics.
\newblock \doi{10.18653/v1/2022.acl-long.132}.
\newblock URL \url{https://aclanthology.org/2022.acl-long.132}.

\bibitem[Mitchell et~al.(2019)Mitchell, Wu, Zaldivar, Barnes, Vasserman,
  Hutchinson, Spitzer, Raji, and Gebru]{mitchell2019model}
Margaret Mitchell, Simone Wu, Andrew Zaldivar, Parker Barnes, Lucy Vasserman,
  Ben Hutchinson, Elena Spitzer, Inioluwa~Deborah Raji, and Timnit Gebru.
\newblock Model cards for model reporting.
\newblock In \emph{Proceedings of the conference on fairness, accountability,
  and transparency}, pages 220--229, 2019.

\bibitem[Stanovsky et~al.(2019)Stanovsky, Smith, and
  Zettlemoyer]{stanovskyetal2019evaluating}
Gabriel Stanovsky, Noah~A. Smith, and Luke Zettlemoyer.
\newblock Evaluating gender bias in machine translation.
\newblock In \emph{Proceedings of the 57th Annual Meeting of the Association
  for Computational Linguistics}, pages 1679--1684, Florence, Italy, July 2019.
  Association for Computational Linguistics.
\newblock \doi{10.18653/v1/P19-1164}.
\newblock URL \url{https://aclanthology.org/P19-1164}.

\bibitem[Ravfogel et~al.(2020{\natexlab{a}})Ravfogel, Elazar, Gonen, Twiton,
  and Goldberg]{ravfogeletal2020null}
Shauli Ravfogel, Yanai Elazar, Hila Gonen, Michael Twiton, and Yoav Goldberg.
\newblock Null it out: Guarding protected attributes by iterative nullspace
  projection.
\newblock In \emph{Proceedings of the 58th Annual Meeting of the Association
  for Computational Linguistics}, pages 7237--7256, Online, July
  2020{\natexlab{a}}. Association for Computational Linguistics.
\newblock \doi{10.18653/v1/2020.acl-main.647}.
\newblock URL \url{https://aclanthology.org/2020.acl-main.647}.

\bibitem[Liang et~al.(2020)Liang, Li, Zheng, Lim, Salakhutdinov, and
  Morency]{liangetal2020towards}
Paul~Pu Liang, Irene~Mengze Li, Emily Zheng, Yao~Chong Lim, Ruslan
  Salakhutdinov, and Louis-Philippe Morency.
\newblock Towards debiasing sentence representations.
\newblock In \emph{Proceedings of the 58th Annual Meeting of the Association
  for Computational Linguistics}, pages 5502--5515, Online, July 2020.
  Association for Computational Linguistics.
\newblock \doi{10.18653/v1/2020.acl-main.488}.
\newblock URL \url{https://aclanthology.org/2020.acl-main.488}.

\bibitem[Bordia and Bowman(2019)]{bordia-bowman-2019-identifying}
Shikha Bordia and Samuel~R. Bowman.
\newblock Identifying and reducing gender bias in word-level language models.
\newblock In \emph{Proceedings of the 2019 Conference of the North {A}merican
  Chapter of the Association for Computational Linguistics: Student Research
  Workshop}, pages 7--15, Minneapolis, Minnesota, June 2019. Association for
  Computational Linguistics.
\newblock \doi{10.18653/v1/N19-3002}.
\newblock URL \url{https://aclanthology.org/N19-3002}.

\bibitem[Qian et~al.(2019)Qian, Muaz, Zhang, and Hyun]{qian-etal-2019-reducing}
Yusu Qian, Urwa Muaz, Ben Zhang, and Jae~Won Hyun.
\newblock Reducing gender bias in word-level language models with a
  gender-equalizing loss function.
\newblock In \emph{Proceedings of the 57th Annual Meeting of the Association
  for Computational Linguistics: Student Research Workshop}, pages 223--228,
  Florence, Italy, July 2019. Association for Computational Linguistics.
\newblock \doi{10.18653/v1/P19-2031}.
\newblock URL \url{https://aclanthology.org/P19-2031}.

\bibitem[Lauscher et~al.(2020)Lauscher, Glava{\v{s}}, Ponzetto, and
  Vuli{\'c}]{lauscher2020general}
Anne Lauscher, Goran Glava{\v{s}}, Simone~Paolo Ponzetto, and Ivan Vuli{\'c}.
\newblock A general framework for implicit and explicit debiasing of
  distributional word vector spaces.
\newblock In \emph{Proceedings of the AAAI Conference on Artificial
  Intelligence}, volume~34, pages 8131--8138, 2020.

\bibitem[Zhang et~al.(2018)Zhang, Lemoine, and Mitchell]{zhang2018mitigating}
Brian~Hu Zhang, Blake Lemoine, and Margaret Mitchell.
\newblock Mitigating unwanted biases with adversarial learning.
\newblock In \emph{Proceedings of the 2018 AAAI/ACM Conference on AI, Ethics,
  and Society}, pages 335--340, 2018.

\bibitem[Yang and Feng(2020)]{yang2020causal}
Zekun Yang and Juan Feng.
\newblock A causal inference method for reducing gender bias in word embedding
  relations.
\newblock In \emph{Proceedings of the AAAI Conference on Artificial
  Intelligence}, volume~34, pages 9434--9441, 2020.

\bibitem[Shin et~al.(2020)Shin, Song, Jang, Kim, Joo, and
  Moon]{shin2020neutralizing}
Seungjae Shin, Kyungwoo Song, JoonHo Jang, Hyemi Kim, Weonyoung Joo, and
  Il-Chul Moon.
\newblock Neutralizing gender bias in word embedding with latent
  disentanglement and counterfactual generation.
\newblock \emph{arXiv preprint arXiv:2004.03133}, 2020.

\bibitem[Pearl(2013)]{pearl2013direct}
Judea Pearl.
\newblock Direct and indirect effects.
\newblock \emph{arXiv preprint arXiv:1301.2300}, 2013.

\bibitem[Ethayarajh(2019)]{ethayarajh-2019-contextual}
Kawin Ethayarajh.
\newblock How contextual are contextualized word representations? {C}omparing
  the geometry of {BERT}, {ELM}o, and {GPT}-2 embeddings.
\newblock In \emph{Proceedings of the 2019 Conference on Empirical Methods in
  Natural Language Processing and the 9th International Joint Conference on
  Natural Language Processing (EMNLP-IJCNLP)}, pages 55--65, Hong Kong, China,
  November 2019. Association for Computational Linguistics.
\newblock \doi{10.18653/v1/D19-1006}.
\newblock URL \url{https://aclanthology.org/D19-1006}.

\bibitem[Radford et~al.(2019)Radford, Wu, Child, Luan, Amodei, Sutskever,
  et~al.]{radford2019language}
Alec Radford, Jeffrey Wu, Rewon Child, David Luan, Dario Amodei, Ilya
  Sutskever, et~al.
\newblock Language models are unsupervised multitask learners.
\newblock \emph{OpenAI blog}, 1\penalty0 (8):\penalty0 9, 2019.

\bibitem[Wolf et~al.(2019)Wolf, Debut, Sanh, Chaumond, Delangue, Moi, Cistac,
  Rault, LoUf, and Funtowicz]{2019huggingFace}
T.~Wolf, L.~Debut, V.~Sanh, J.~Chaumond, C.~Delangue, A.~Moi, P.~Cistac,
  T.~Rault, R.~LoUf, and M.~Funtowicz.
\newblock Huggingface's transformers: State-of-the-art natural language
  processing.
\newblock 2019.

\bibitem[Ravfogel et~al.(2020{\natexlab{b}})Ravfogel, Elazar, Gonen, Twiton,
  and Goldberg]{ravfogel2020null}
Shauli Ravfogel, Yanai Elazar, Hila Gonen, Michael Twiton, and Yoav Goldberg.
\newblock Null it out: Guarding protected attributes by iterative nullspace
  projection.
\newblock \emph{arXiv preprint arXiv:2004.07667}, 2020{\natexlab{b}}.

\bibitem[Nadeem et~al.(2021)Nadeem, Bethke, and
  Reddy]{nadeem-etal-2021-stereoset}
Moin Nadeem, Anna Bethke, and Siva Reddy.
\newblock {S}tereo{S}et: Measuring stereotypical bias in pretrained language
  models.
\newblock In \emph{Proceedings of the 59th Annual Meeting of the Association
  for Computational Linguistics and the 11th International Joint Conference on
  Natural Language Processing (Volume 1: Long Papers)}, pages 5356--5371,
  Online, August 2021. Association for Computational Linguistics.
\newblock \doi{10.18653/v1/2021.acl-long.416}.
\newblock URL \url{https://aclanthology.org/2021.acl-long.416}.

\bibitem[May et~al.(2019)May, Wang, Bordia, Bowman, and
  Rudinger]{may-etal-2019-measuring}
Chandler May, Alex Wang, Shikha Bordia, Samuel~R. Bowman, and Rachel Rudinger.
\newblock On measuring social biases in sentence encoders.
\newblock In \emph{Proceedings of the 2019 Conference of the North {A}merican
  Chapter of the Association for Computational Linguistics: Human Language
  Technologies, Volume 1 (Long and Short Papers)}, pages 622--628, Minneapolis,
  Minnesota, June 2019. Association for Computational Linguistics.
\newblock \doi{10.18653/v1/N19-1063}.
\newblock URL \url{https://aclanthology.org/N19-1063}.

\bibitem[Merity et~al.(2016)Merity, Xiong, Bradbury, and
  Socher]{Merity2016Pointer}
S.~Merity, C.~Xiong, J.~Bradbury, and R.~Socher.
\newblock Pointer sentinel mixture models.
\newblock 2016.

\bibitem[Wang et~al.(2018)Wang, Singh, Michael, Hill, Levy, and
  Bowman]{wangetal2018glue}
Alex Wang, Amanpreet Singh, Julian Michael, Felix Hill, Omer Levy, and Samuel
  Bowman.
\newblock {GLUE}: A multi-task benchmark and analysis platform for natural
  language understanding.
\newblock In \emph{Proceedings of the 2018 {EMNLP} Workshop {B}lackbox{NLP}:
  Analyzing and Interpreting Neural Networks for {NLP}}, pages 353--355,
  Brussels, Belgium, November 2018. Association for Computational Linguistics.
\newblock \doi{10.18653/v1/W18-5446}.
\newblock URL \url{https://aclanthology.org/W18-5446}.

\bibitem[Caliskan et~al.(2017)Caliskan, Bryson, and
  Narayanan]{caliskan2017semantics}
Aylin Caliskan, Joanna~J Bryson, and Arvind Narayanan.
\newblock Semantics derived automatically from language corpora contain
  human-like biases.
\newblock \emph{Science}, 356\penalty0 (6334):\penalty0 183--186, 2017.

\bibitem[Gonen and Goldberg(2019)]{gonengoldberg2019lipstick}
Hila Gonen and Yoav Goldberg.
\newblock Lipstick on a pig: {D}ebiasing methods cover up systematic gender
  biases in word embeddings but do not remove them.
\newblock In \emph{Proceedings of the 2019 Conference of the North {A}merican
  Chapter of the Association for Computational Linguistics: Human Language
  Technologies, Volume 1 (Long and Short Papers)}, pages 609--614, Minneapolis,
  Minnesota, June 2019. Association for Computational Linguistics.
\newblock \doi{10.18653/v1/N19-1061}.
\newblock URL \url{https://aclanthology.org/N19-1061}.

\end{thebibliography}

\end{document}